\documentclass{article}

\usepackage{microtype}
\usepackage{graphicx}
\usepackage{subfigure}
\usepackage{booktabs}

\usepackage{bbm}
\usepackage{subfigure}
\usepackage{multirow}
\usepackage{color,verbatim}
\usepackage[tableposition=top]{caption}
\usepackage{hyperref}
\usepackage{enumitem}
\usepackage{amssymb}
\usepackage{amsmath}
\usepackage{amsthm}

\usepackage{blindtext}

\newtheorem{theorem}{Theorem}[section]
\theoremstyle{definition}

\newcommand{\defeq}{\overset{\text{\tiny def}}{=}}

\usepackage{array}
\newcolumntype{L}[1]{>{\raggedright\let\newline\\\arraybackslash\hspace{0pt}}m{#1}}
\newcolumntype{C}[1]{>{\centering\let\newline\\\arraybackslash\hspace{0pt}}m{#1}}
\newcolumntype{R}[1]{>{\raggedleft\let\newline\\\arraybackslash\hspace{0pt}}m{#1}}

\def\method{GMNN}

\usepackage{hyperref}

\usepackage[accepted]{icml2019}

\icmltitlerunning{GMNN: Graph Markov Neural Networks}

\begin{document}

\twocolumn[
\icmltitle{GMNN: Graph Markov Neural Networks}

\icmlsetsymbol{equal}{*}

\begin{icmlauthorlist}
\icmlauthor{Meng Qu}{1,2}
\icmlauthor{Yoshua Bengio}{1,2,4}
\icmlauthor{Jian Tang}{1,4,3}
\end{icmlauthorlist}

\icmlaffiliation{1}{Mila - Qu\'ebec AI Institute}
\icmlaffiliation{2}{University of Montr\'eal}
\icmlaffiliation{3}{HEC Montr\'eal}
\icmlaffiliation{4}{Canadian Institute for Advanced Research (CIFAR)}

\icmlcorrespondingauthor{Meng Qu}{meng.qu@umontreal.ca}
\icmlcorrespondingauthor{Jian Tang}{jian.tang@hec.ca}

\icmlkeywords{Graph Neural Network, Statistical Relational Learning.}

\vskip 0.3in
]

\printAffiliationsAndNotice{}

\begin{abstract}

This paper studies semi-supervised object classification in relational data, which is a fundamental problem in relational data modeling. The problem has been extensively studied in the literature of both statistical relational learning (e.g. relational Markov networks) and graph neural networks (e.g. graph convolutional networks). Statistical relational learning methods can effectively model the dependency of object labels through conditional random fields for collective classification, whereas graph neural networks learn effective object representations for classification through end-to-end training. In this paper, we propose the Graph Markov Neural Network (\method{}) that combines the advantages of both worlds. A \method{} models the joint distribution of object labels with a conditional random field, which can be effectively trained with the variational EM algorithm. In the E-step, one graph neural network learns effective object representations for approximating the posterior distributions of object labels. In the M-step, another graph neural network is used to model the local label dependency. Experiments on object classification, link classification, and unsupervised node representation learning show that \method{} achieves state-of-the-art results. 

\end{abstract}

\section{Introduction}

We live in an interconnected world, where entities are connected through various relations. For example, web pages are linked by hyperlinks; social media users are connected through friendship relations. Modeling such relational data is an important topic in machine learning, covering a variety of applications such as entity classification~\cite{perozzi2014deepwalk}, link prediction~\cite{taskar2004link} and link classification~\cite{dettmers2018convolutional}.

Many of these applications can be boiled down to the fundamental problem of semi-supervised object classification~\cite{taskar2007relational}. Specifically, objects~\footnote{In this paper, we will use ``object'' and ``node'' interchangeably to refer to entities in graphs, because they are different terminologies used in the literature of statistical relational learning and graph neural networks.} are interconnected and associated with some attributes. Given the labels of a few objects, the goal is to infer the labels of other objects.
This problem has been extensively studied in the literature of statistical relational learning (SRL), which develops statistical methods to model relational data. Some representative methods include relational Markov networks (RMN)~\cite{taskar2002discriminative} and Markov logic networks (MLN)~\cite{richardson2006markov}. Generally, these methods model the dependency of object labels using conditional random fields~\cite{lafferty2001conditional}. Because of their effectiveness for modeling label dependencies, these methods achieve compelling results on semi-supervised object classification. However, several limitations still remain. (1) These methods typically define potential functions in conditional random fields as linear combinations of some hand-crafted feature functions, which are quite heuristic. Moreover, the capacity of such models is usually insufficient. (2) Due to the complexity of relational structures between objects, inferring the posterior distributions of object labels for unlabeled objects remains a challenging problem.

Another line of research is based on the recent progress of graph neural networks~\cite{kipf2016semi,hamilton2017inductive,gilmer2017neural,velivckovic2018graph}. Graph neural networks approach object classification by learning effective object representations with non-linear neural architectures, and the whole framework can be trained in an end-to-end fashion. For example, the graph convolutional network (GCN)~\cite{kipf2016semi} iteratively updates the representation of each object by combining its own representation and the representations of the surrounding objects. These approaches have been shown to achieve state-of-the-art performance because of their effectiveness in learning object representations on relational data. However, one critical limitation is that the labels of objects are independently predicted based on their representations. In other words, the joint dependency of object labels is ignored.

In this paper, we propose a new approach called the Graph Markov Neural Network (\method{}), which combines the advantages of both statistical relational learning and graph neural networks. A \method{} is able to learn effective object representations as well as model label dependency between different objects. Similar to SRL methods, a \method{} includes a conditional random field~\cite{lafferty2001conditional} to model the joint distribution of object labels conditioned on object attributes. This framework can be effectively and efficiently optimized with the variational EM framework~\cite{neal1998view}, alternating between an inference procedure (E-step) and a learning procedure (M-step). In the learning procedure, instead of maximizing the likelihood function, the training procedure for \method{}s optimizes the pseudolikelihood function~\cite{besag1975statistical} and parameterizes the local conditional distributions of object labels with a graph neural network. Such a graph neural network can well model the dependency of object labels, and no hand-crafted potential functions are required. For inference, since exact inference is intractable, we use a mean-field approximation~\cite{opper2001advanced}. Inspired by the idea of amortized inference~\cite{gershman2014amortized,kingma2013auto}, we further parameterize the posterior distributions of object labels with another graph neural network, which is able to learn useful object representations for predicting object labels. With a graph neural network for inference, the number of parameters can be significantly reduced, and the statistical evidence can be shared across different objects in inference~\cite{kingma2013auto}. 

Our \method{} approach is very general. Though it is designed for object classification, it can be naturally applied to many other applications, such as unsupervised node representation learning and link classification. Experiment results show that \method{}s achieve state-of-the-art results on object classification and unsupervised node representation learning, as well as very competitive results on link classification.

\section{Related Work}

\smallskip
\noindent \textbf{Statistical Relational Learning.} In the literature of statistical relational learning (SRL), a variety of methods have been proposed for semi-supervised object classification. The basic idea is to model label dependency with probabilistic graphical models. Many early methods~\cite{koller1998probabilistic,friedman1999learning,getoor2001learning,getoor2001probabilistic,xiang2008pseudolikelihood} are built on top of directed graphical models. However, these methods can only handle acyclic dependencies among objects, and their performance for prediction is usually limited. Due to such weaknesses, many later SRL methods employ Markov networks (e.g., conditional random fields~\cite{lafferty2001conditional}), and representative methods include relational Markov networks (RMN)~\cite{taskar2002discriminative} and Markov logic networks (MLN)~\cite{richardson2006markov,singla2005discriminative}. Though these methods are quite effective, they still suffer from several challenges. (1) Some hand-crafted feature functions are required for specifying the potential function, and the whole framework is designed as a log-linear model by combining different feature functions, so the capacity of such frameworks is quite limited. (2) Inference remains challenging due to the complicated relational structures among objects. Our proposed \method{} method overcomes the above challenges by using two different graph neural networks, one for modeling the label dependency and another for approximating the posterior label distributions, and the approach can be effectively trained with the variational EM algorithm. 

\smallskip
\noindent \textbf{Graph-based Semi-supervised Classification.} Another category of related work is graph-based semi-supervised classification. For example, the label propagation methods~\cite{zhu2003semi,zhou2004learning} iteratively propagate the label of each object to its neighbors. However, these methods can only model the linear dependency of object labels, while in our approach a non-linear graph neural network is used to model label dependency, which has greater expressive power. Moreover, \method{}s can also learn useful object representations for predicting object labels. 

\smallskip
\noindent \textbf{Graph Neural Networks.} Another closely related research area is that of graph neural networks~\cite{dai2016discriminative,kipf2016semi,hamilton2017inductive,gilmer2017neural,velivckovic2018graph}, which can learn useful object representations for predicting object labels. Essentially, the object representations are learned by encoding local graph structures and object attributes, and the whole framework can be trained in an end-to-end fashion.
Because of their effectiveness in learning object representations, they achieve state-of-the-art results in object classification. However, existing methods usually ignore the dependency between object labels. With \method{}s, besides learning object representations, we also model the joint dependency of object labels by introducing conditional random fields. 

\smallskip
\noindent \textbf{GNN for PGM Inference.} There are also some recent studies~\cite{yoon2018inference} using graph neural networks for inference in probabilistic graphical models. Compared with their work, our work focuses on statistical relational learning while their work puts more emphasis on standard graphical models. Moreover, our work utilizes two graph neural networks for both inference and learning, while their work only uses one graph neural network for inference.
\section{Problem Definition \& Preliminary}

\subsection{Problem Definition}

Relational data, in which objects are interconnected via different relations, are ubiquitous in the real world. Modeling relational data is an important direction in machine learning with various applications, such as object classification and link prediction. In this paper, we focus on a fundamental problem, semi-supervised object classification, as many other applications can be reformulated into this problem.

Formally, the problem of semi-supervised object classification considers a graph $G=(V,E,\mathbf{x}_V)$, in which $V$ is a set of objects, $E$ is a set of edges between objects, and $\mathbf{x}_V$ stands for the attributes of all the objects. The edges in $E$ may have multiple types, which represent different relations among objects. In this paper, for simplicity, we assume all edges belong to the same type. Given the labels $\mathbf{y}_L$ of a few labeled objects $L \subset V$, the goal is to predict the labels $\mathbf{y}_U$ for the remaining unlabeled objects $U = V \setminus L$.

This problem has been extensively studied in the literature of both statistical relation learning (SRL) and graph neural networks (GNN). Essentially, both types of methods aim to model the distribution of object labels conditioned on the object attributes and the graph structure, i.e. $p(\mathbf{y}_V|\mathbf{x}_V,E)$. Next, we introduce the general idea of both methods. For notation simplicity, we omit $E$ in the following formulas. 

\subsection{Statistical Relational Learning}
Most SRL methods model $p(\mathbf{y}_V|\mathbf{x}_V)$ with conditional random fields, which employ the following formulation:
\begin{equation}
\begin{aligned}
\label{eqn::obj-srl}
p(\mathbf{y}_V|\mathbf{x}_V)=\frac{1}{Z(\mathbf{x}_V)} \prod_{(n_i,n_j) \in E} \psi_{i,j}(\mathbf{y}_{n_i}, \mathbf{y}_{n_j},\mathbf{x}_V).
\end{aligned}
\end{equation}
Here, $(n_i,n_j)$ is an edge in graph $G$, and $\psi_{i,j}(\mathbf{y}_{n_i}, \mathbf{y}_{n_j},\mathbf{x}_V)$ is the potential score defined on the edge. Typically, the potential score is computed as a linear combination of some hand-crafted feature functions, such as logical formulae.

With this formulation, predicting the labels for unlabeled objects becomes an inference problem, i.e., inferring the posterior label distribution of the unlabeled objects $p(\mathbf{y}_U|\mathbf{y}_L,\mathbf{x}_V)$. Exact inference is usually infeasible due to the complicated structures between object labels. Therefore, some approximation inference methods are often utilized, such as loopy belief propagation~\cite{murphy1999loopy}.

\subsection{Graph Neural Network}
Different from SRL methods, GNN methods simply ignore the dependency of object labels and they focus on learning effective object representations for label prediction. Specifically, the joint distribution of labels is fully factorized as:
\begin{equation}
\begin{aligned}
\label{eqn::obj-gnn}
p(\mathbf{y}_V|\mathbf{x}_V)=\prod_{n \in V} p(\mathbf{y}_n|\mathbf{x}_V).
\end{aligned}
\end{equation}
Based on the formulation, GNNs will infer the label distribution $p(\mathbf{y}_n|\mathbf{x}_V)$ for each object $n$ independently. For each object $n$, GNNs predict the label in the following way:
\begin{equation}\nonumber
\begin{aligned}
\mathbf{h} = g(\mathbf{x}_V,E) \quad 
p(\mathbf{y}_n|\mathbf{x}_V) = \text{Cat}(\mathbf{y}_n|\text{softmax}(W \mathbf{h}_n)),
\end{aligned}
\end{equation}
where $\mathbf{h} \in \mathbb{R}^{|V| \times d}$ is the representations of all the objects, and $\mathbf{h}_n \in \mathbb{R}^d$ is the representation of object $n$. $W\in \mathbb{R}^{K \times d}$ is a linear transformation matrix, with $d$ as the representation dimension and $K$ as the number of label classes. $\text{Cat}$ stands for categorical distributions. Basically, GNNs focus on learning a useful representation $\mathbf{h}_n$ for each object $n$. Specifically, each $\mathbf{h}_n$ is initialized as the attribute representation of object $n$. Then each $\mathbf{h}_n$ is iteratively updated according to its current value and the representations of $n$'s neighbors, i.e. $\mathbf{h}_{\text{NB}(n)}$. For the updating function, the graph convolutional layer (GC)~\cite{kipf2016semi} and the graph attention layer (GAT)~\cite{velivckovic2018graph} can be used, or in general the neural message passing layer~\cite{gilmer2017neural} can be utilized. After multiple layers of update, the final object representations are fed into a linear softmax classifier for label prediction. The whole framework can be trained in an end-to-end fashion with a few labeled objects.

\section{GMNN: Graph Markov Neural Network}

\begin{figure}
	\centering
	\includegraphics[width=0.45\textwidth]{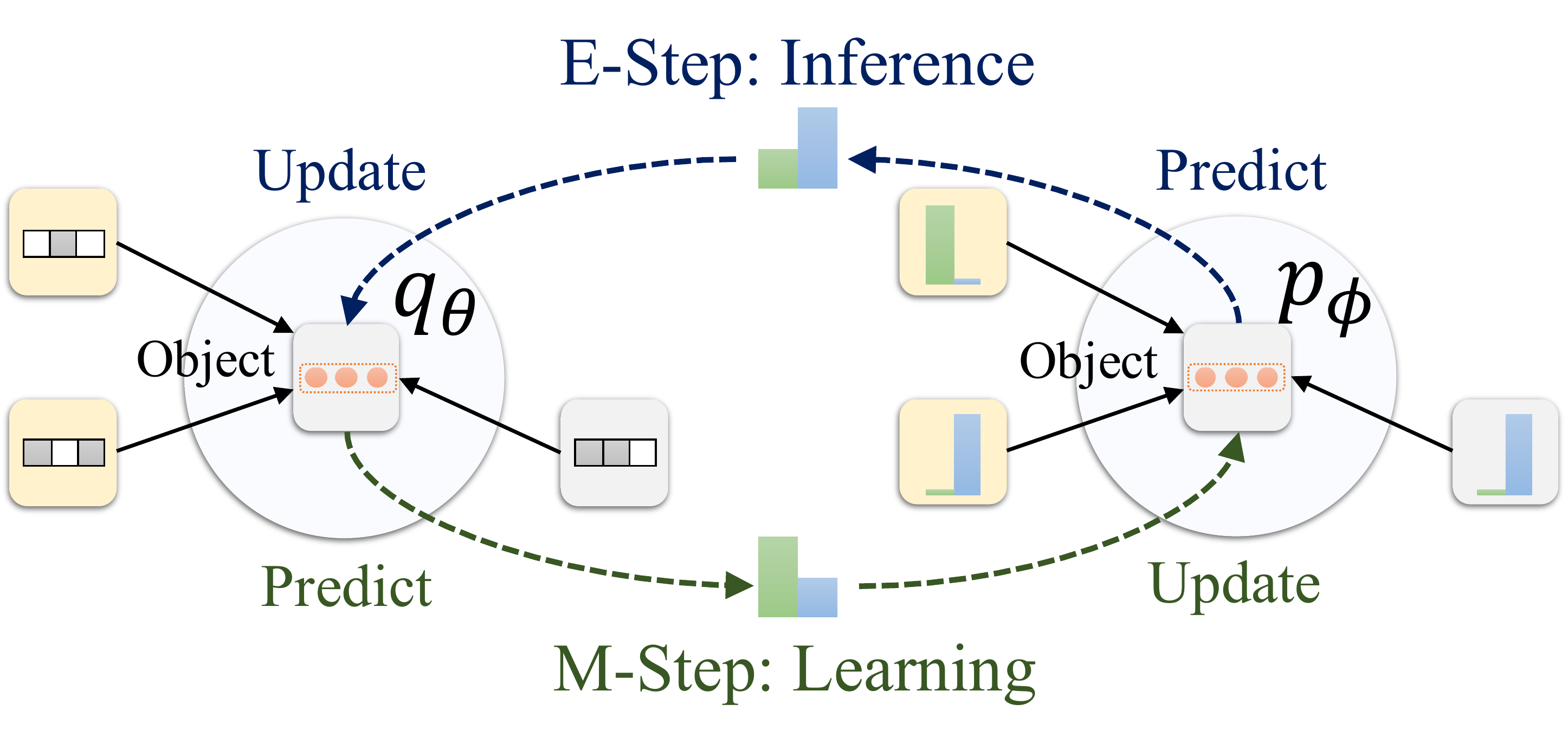}
	\caption{Framework overview. Yellow and grey squares are labeled and unlabeled objects. Grey/white grids are attributes. Histograms are label distributions of objects. Orange triple circles are object representations. \method{} is trained by alternating between an E-step and an M-step. See Sec.~\ref{sec::optim} for the detailed explanation.
	}
	\label{fig::framework}
\end{figure}

In this section, we introduce our approach called the Graph Markov Neural Network (GMNN) for semi-supervised object classification. The goal of GMNN is to combine the advantages of both the statistical relational learning methods and graph neural networks, such that we can learn useful objective representations for predicting object labels, as well as model the dependency between object labels. Specifically, GMNN models the joint distribution of object labels conditioned on object attributes, i.e. $p(\mathbf{y}_V|\mathbf{x}_V)$, by using a conditional random field, which is optimized with a pseudolikelihood variational EM framework. In the E-step, a graph neural network is used to learn object representations for label prediction. In the M-step, another graph neural network is employed to model the local dependency of object labels. Next, we introduce the details of the \method{} approach.

\subsection{Pseudolikelihood Variational EM}
Following existing SRL methods, we use a conditional random field as in Eq.~\eqref{eqn::obj-srl} to model the joint distribution of object labels conditioned on object attributes, i.e. $p_{\phi}(\mathbf{y}_V|\mathbf{x}_V)$, where the potential is defined over each edge, and $\phi$ is the model parameters. For now, we ignore the specific formulation of the potential function, and we will discuss it later.

We learn the model parameters $\phi$ by maximizing the log-likelihood function of the observed object labels, i.e. $\log p_{\phi}(\mathbf{y}_L|\mathbf{x}_V)$. However, directly maximizing the log-likelihood function is difficult, since many object labels are unobserved. Therefore, we instead optimize the evidence lower bound (ELBO) of the log-likelihood function:
\begin{equation}
\label{eqn::elbo}
\begin{aligned}
& \log p_\phi(\mathbf{y}_L|\mathbf{x}_V) \geq \\ & \mathbb{E}_{q_\theta(\mathbf{y}_U|\mathbf{x}_V)}[\log p_\phi(\mathbf{y}_L,\mathbf{y}_U|\mathbf{x}_V) - \log q_\theta(\mathbf{y}_U|\mathbf{x}_V)],
\end{aligned}
\end{equation}
where $q_\theta(\mathbf{y}_U|\mathbf{x}_V)$ can be any distributions over $\mathbf{y}_U$, and the equation holds when $q_\theta(\mathbf{y}_U|\mathbf{x}_V)=p_\phi(\mathbf{y}_U|\mathbf{y}_L,\mathbf{x}_V)$. According to the variational EM algorithm~\cite{neal1998view}, such a lower bound can be optimized by alternating between a variational E-step and an M-step. In the variational E-step (a.k.a., inference procedure), the goal is to fix $p_\phi$ and update the variational distribution $q_\theta(\mathbf{y}_U|\mathbf{x}_V)$ to approximate the true posterior distribution $p_\phi(\mathbf{y}_U|\mathbf{y}_L,\mathbf{x}_V)$.

In the M-step (a.k.a., learning procedure), we fix $q_\theta$ and update $p_\phi$ to maximize the likelihood function below:
\begin{equation}
\label{eqn::obj-p-likelihood}
\begin{aligned}
\ell(\phi) = \mathbb{E}_{q_\theta(\mathbf{y}_U|\mathbf{x}_V)}[\log p_\phi(\mathbf{y}_L,\mathbf{y}_U|\mathbf{x}_V)].
\end{aligned}
\end{equation}
However, directly optimizing the likelihood function can be difficult, as we have to deal with the partition function in $p_\phi$. To avoid computing the partition function, we instead optimize the pseudolikelihood function~\cite{besag1975statistical} below:
\begin{equation}
\label{eqn::obj-p-pseudo}
\begin{aligned}
\ell_{PL}(\phi) &\triangleq \mathbb{E}_{q_\theta(\mathbf{y}_U|\mathbf{x}_V)}[\sum_{n \in V} \log p_\phi(\mathbf{y}_n|\mathbf{y}_{V \setminus n},\mathbf{x}_V)] \\
&=\mathbb{E}_{q_\theta(\mathbf{y}_U|\mathbf{x}_V)}[\sum_{n \in V} \log p_\phi(\mathbf{y}_n|\mathbf{y}_{\text{NB(n)}},\mathbf{x}_V)],
\end{aligned}
\end{equation}
where $\text{NB}(n)$ is the neighbor set of $n$, and the equation is based on the independence properties of $p_{\phi}(\mathbf{y}_V|\mathbf{x}_V)$ derived from its formulation, i.e. Eq.~\eqref{eqn::obj-srl}. The pseudolikelihood approach is widely used for learning Markov networks~\cite{kok2005learning,richardson2006markov}.
Next, we introduce the details of the inference and learning steps.

\subsection{Inference}
The inference step aims to compute the posterior distribution $p_\phi(\mathbf{y}_U|\mathbf{y}_L,\mathbf{x}_V)$. Due to the complicated relational structures between object labels, exact inference is intractable. Therefore, we approximate it with another variational distribution $q_\theta(\mathbf{y}_U|\mathbf{x}_V)$. Specifically, we use the mean-field method~\cite{opper2001advanced}, in which $q_\theta$ is formulated as:
\begin{equation}
\begin{aligned}
q_\theta(\mathbf{y}_U|\mathbf{x}_V)=\prod_{n \in U} q_\theta(\mathbf{y}_n|\mathbf{x}_V).
\end{aligned}
\end{equation}
Here, $n$ is the index of unlabeled objects. In the variational distribution, all object labels are assumed to be independent.

To model the distribution of each object label in $q_\theta$, we follow the idea of amortized inference~\cite{gershman2014amortized,kingma2013auto}, and parameterize $q_\theta(\mathbf{y}_n|\mathbf{x}_V)$ with a graph neural network (GNN), which learns effective object representations for label prediction:
\begin{equation}
\begin{aligned}
q_\theta(\mathbf{y}_n|\mathbf{x}_V)= \text{Cat}(\mathbf{y}_n|\text{softmax}(W_\theta\mathbf{h}_{\theta,n})).
\end{aligned}
\end{equation}
Specifically, $q_\theta(\mathbf{y}_n|\mathbf{x}_V)$ is formulated as a categorical distribution, and the probability of each class is calculated by a softmax classifier based on the object representation $\mathbf{h}_{\theta,n}$. The representation $\mathbf{h}_{\theta,n}$ is learned by a GNN model with the object attributes $\mathbf{x}_V$ as features, and $\theta$ as parameters. We denote the GNN model as $\text{GNN}_\theta$. With $\text{GNN}_\theta$, we can improve inference by learning useful representations of objects from their attributes and local connections. Besides, by sharing $\text{GNN}_\theta$ across different objects, we can significantly reduce the number of parameters required for inference, which is more efficient~\cite{kingma2013auto}.

With the above mean-field formulation, if we fix distribution $q_\theta(\mathbf{y}_{\text{NB}(n) \cap U}|\mathbf{x}_V)$, then the optimum of $q_\theta(\mathbf{y}_n|\mathbf{x}_V)$, denoted by $q^*(\mathbf{y}_n|\mathbf{x}_V)$, is specified by the following condition:
\begin{equation}
\begin{aligned}
\label{eqn::optimal-condition}
&\log q^*(\mathbf{y}_n|\mathbf{x}_V) =\\ & \mathbb{E}_{ q_\theta(\mathbf{y}_{\text{NB}(n) \cap U}|\mathbf{x}_V)}[\log p_\phi(\mathbf{y}_n|\mathbf{y}_{\text{NB}(n)},\mathbf{x}_V)]
+ \text{const}.
\end{aligned}
\end{equation}
See appendix for the proof. Bases on that, for each object $n$, we optimize $q_\theta(\mathbf{y}_n|\mathbf{x}_V)$ with a method similar to~\citet{salakhutdinov2010efficient}. More specifically, we start by using $q_\theta(\mathbf{y}_{\text{NB}(n) \cap U}|\mathbf{x}_V)$ to compute $q^*(\mathbf{y}_n|\mathbf{x}_V)$, which is further treated as target to update $q_\theta(\mathbf{y}_n|\mathbf{x}_V)$. Computing $q^*(\mathbf{y}_n|\mathbf{x}_V)$ in Eq.~\eqref{eqn::optimal-condition} relies on computing the expectation with respect to $q_\theta(\mathbf{y}_{\text{NB}(n) \cap U}|\mathbf{x}_V)$. We estimate the expectation by drawing a sample from $q_\theta(\mathbf{y}_{\text{NB}(n) \cap U}|\mathbf{x}_V)$, yielding:
\begin{equation}
\begin{aligned}
\label{eqn::optimal-condition2}
&\mathbb{E}_{ q_\theta(\mathbf{y}_{\text{NB}(n) \cap U}|\mathbf{x}_V)}[\log p_\phi(\mathbf{y}_n|\mathbf{y}_{\text{NB}(n)},\mathbf{x}_V)] \\
&\simeq \log p_\phi(\mathbf{y}_n|\mathbf{\hat{y}}_{\text{NB}(n)},\mathbf{x}_V).
\end{aligned}
\end{equation}
In the above formula, $\mathbf{\hat{y}}_{\text{NB}(n)}=\{ \mathbf{\hat{y}}_{n'} \}_{n' \in \text{NB}(n)} $ is defined as below. For each unlabeled neighbor $n'$ of object $n$, we sample $\mathbf{\hat{y}}_{n'} \sim q_\theta(\mathbf{y}_{n'}|\mathbf{x}_V)$, and for each labeled neighbor $n'$ of object $n$, $\mathbf{\hat{y}}_{n'}$ is set as the ground-truth label. In practice, we find that using one sample from $q_\theta(\mathbf{y}_{\text{NB}(n) \cap U}|\mathbf{x}_V)$ yields comparable results with multiple samples. Therefore, in the experiments, only one sample is used for efficiency purpose. 

Given Eq.~\eqref{eqn::optimal-condition} and~\eqref{eqn::optimal-condition2}, $q^*(\mathbf{y}_n|\mathbf{x}_V)$ can thus be approximated as $q^*(\mathbf{y}_n|\mathbf{x}_V) \approx p_\phi(\mathbf{y}_n|\mathbf{\hat{y}}_{\text{NB}(n)},\mathbf{x}_V)$. Therefore, we could instead treat $p_\phi(\mathbf{y}_n|\mathbf{\hat{y}}_{\text{NB}(n)},\mathbf{x}_V)$ as target, and minimize $\text{KL}(p_\phi(\mathbf{y}_n|\mathbf{\hat{y}}_{\text{NB}(n)},\mathbf{x}_V) || q_\theta(\mathbf{y}_n|\mathbf{x}_V))$. We further use a parallel update strategy~\cite{koller2009probabilistic} to speed up training, where we jointly optimize $q_\theta(\mathbf{y}_n|\mathbf{x}_V)$ for every unlabeled object $n$, yielding the objective as follows:
\vspace{-0.1cm}
\begin{equation}
\label{eqn::obj-q-u}
\begin{aligned}
O_{\theta,U} = \sum_{n \in U} \mathbb{E}_{p_\phi(\mathbf{y}_n|\mathbf{\hat{y}}_{\text{NB}(n)},\mathbf{x}_V)}[\log q_\theta(\mathbf{y}_n|\mathbf{x}_V)].
\end{aligned}
\end{equation}
Besides, we notice that $q_\theta$ can be also trained by predicting the labels for the labeled objects. Therefore, we also let $q_\theta$ maximize the following supervised objective function:
\vspace{-0.1cm}
\begin{equation}
\label{eqn::obj-q-l}
\begin{aligned}
O_{\theta,L} = \sum_{n \in L} \log q_\theta(\mathbf{y}_n|\mathbf{x}_V).
\end{aligned}
\end{equation}
Here, $\mathbf{y}_n$ is the ground-truth label of $n$.
By adding Eq.~\eqref{eqn::obj-q-u} and~\eqref{eqn::obj-q-l}, we obtain the overall objective for optimizing $\theta$:
\begin{equation}
\label{eqn::obj-q}
\begin{aligned}
O_{\theta} = O_{\theta,U} + O_{\theta,L}.
\end{aligned}
\end{equation}
\subsection{Learning}
In the M-step, we seek to learn the parameter $\phi$. More specifically, we will fix $q_\theta$ 
and further update $p_\phi$ to maximize Eq.~\eqref{eqn::obj-p-pseudo}. With the objective function, we notice that only the conditional distribution $p_\phi(\mathbf{y}_n|\mathbf{y}_{\text{NB}(n)},\mathbf{x}_V)$ is required for $p_\phi$ in both the inference and learning steps (Eq.~\eqref{eqn::obj-q-u} and (\ref{eqn::obj-p-pseudo})). Therefore, instead of defining the joint distribution of object labels $p_\phi(\mathbf{y}_V|\mathbf{x}_V)$ by specifying the potential function, we can simply focus on modeling the conditional distribution. Here, we parameterize the conditional distribution  $p_\phi(\mathbf{y}_n|\mathbf{y}_{\text{NB}(n)},\mathbf{x}_V)$ with another non-linear graph neural network model (GNN) because of its effectiveness:
\begin{equation}
\begin{aligned}
p_\phi(\mathbf{y}_n|\mathbf{y}_{\text{NB}(n)},\mathbf{x}_V)=\text{Cat}(\mathbf{y}_n|\text{softmax}(W_\phi\mathbf{h}_{\phi,n})).
\end{aligned}
\end{equation}
Here, the distribution of $\mathbf{y}_n$ is characterized by a softmax classifier, which takes the object representation $\mathbf{h}_{\phi,n}$ learned by a GNN model as features, and we denote the GNN as $\text{GNN}_\phi$. When learning the object representation $\mathbf{h}_{\phi,n}$, $\text{GNN}_\phi$ treats all the labels $\mathbf{y}_{\text{NB}(n)}$ surrounding the object $n$ as features. Therefore, $\text{GNN}_\phi$ essentially models local dependencies of object labels. With the above formulation, we no longer require any hand-crafted feature functions. 

The framework is related to the label propagation methods~\cite{zhu2003semi,zhou2004learning}, which also update each object label by combining the surrounding labels. However, these methods propagate labels in a fixed and linear way, whereas $\text{GNN}_\phi$ is in a learnable and non-linear way.

One notable thing is that when defining $p_\phi(\mathbf{y}_n|\mathbf{y}_{\text{NB}(n)},\mathbf{x}_V)$, $\text{GNN}_\phi$ only uses the object labels $\mathbf{y}_{\text{NB}(n)}$ surrounding the object $n$ as features, but $\text{GNN}_\phi$ is flexible to incorporate other features. For example, we can follow existing SRL methods, and take both the surrounding object labels $\mathbf{y}_{\text{NB}(n)}$ and surrounding attributes $\mathbf{x}_{\text{NB}(n)}$ as features in $\text{GNN}_\phi$. We will discuss this variant in our experiment (see Sec.~\ref{sec::experiment-main}).

Another thing is that based on the overall formulation of $p_\phi$, i.e. Eq.~\eqref{eqn::obj-srl}, each object label $\mathbf{y}_n$ should only depend on its adjacent object labels $\mathbf{y}_{\text{NB}(n)}$ and object attributes $\mathbf{x}_V$, which implies $\text{GNN}_\phi$ should not have more than one message passing layer. However, a common practice in the literature of graph neural networks is to use multiple message passing layers during training, which can well model the long-range dependency between different objects. Therefore, we also explore using multiple message passing layers to capture such long-range dependency (see Sec.~\ref{sec::experiment-learn}).

When optimizing $p_\phi$ to maximize Eq.~\eqref{eqn::obj-p-pseudo}, we estimate the expectation in Eq.~\eqref{eqn::obj-p-pseudo} by drawing a sample from $q_\theta(\mathbf{y}_U|\mathbf{x}_V)$. More specifically, if $n$ is an unlabeled object, then we sample $\mathbf{\hat{y}}_n \sim q_\theta(\mathbf{y}_n|\mathbf{x}_V)$, and otherwise we set $\mathbf{\hat{y}}_n$ as the ground-truth label. Afterwards, the parameter $\phi$ can be optimized by maximizing the following objective function:
\begin{equation}
\label{eqn::obj-p}
\begin{aligned}
O_\phi=\sum_{n \in V} \log p_\phi(\mathbf{\hat{y}}_n|\mathbf{\hat{y}}_{\text{NB(n)}},\mathbf{x}_V).
\end{aligned}
\end{equation}

\begin{algorithm}[bt]
    \caption{Optimization Algorithm}
    \label{alg::optim}
    \begin{algorithmic}
        \STATE {\bfseries Input:} A graph $G$, some labeled objects $(L,\mathbf{y}_L)$.
        \STATE {\bfseries Output:} Object labels $\mathbf{y}_U$ for unlabeled objects $U$.
        \STATE Pre-train $q_\theta$ with $\mathbf{y}_L$ according to Eq.~\eqref{eqn::obj-q-l}.
        \WHILE{not converge}
            \STATE $\boxdot$ \emph{M-Step: Learning Procedure}
            \STATE Annotate unlabeled objects with $q_\theta$.
            \STATE Denote the sampled labels as $\mathbf{\hat{y}}_U$.
            \STATE Set $\mathbf{\hat{y}}_V=(\mathbf{y}_L,\mathbf{\hat{y}}_U)$ and update $p_\phi$ with Eq.~\eqref{eqn::obj-p}.
            \STATE $\boxdot$ \emph{E-Step: Inference Procedure}
            \STATE Annotate unlabeled objects with $p_\phi$ and $\mathbf{\hat{y}}_V$.
            \STATE Denote the predicted label distribution as $p_\phi(\mathbf{y}_U)$.
            \STATE Update $q_\theta$ with Eq.~\eqref{eqn::obj-q-u},~\eqref{eqn::obj-q-l} based on $p_\phi(\mathbf{y}_U), \mathbf{y}_L$.
        \ENDWHILE
        \STATE Classify each unlabeled object $n$ based on $q_\theta(\mathbf{y}_n|\mathbf{x}_V)$.
\end{algorithmic}
\end{algorithm}

\vspace{-0.5cm}
\subsection{Optimization}
\label{sec::optim}
To optimize our approach, we pre-train $q_\theta$ with the labeled objects. Then we alternatively optimize $p_\phi$ and $q_\theta$ until convergence. Afterwards, both $p_\phi$ and $q_\theta$ can be used to classify unlabeled objects. In practice, $q_\theta$ consistently outperforms $p_\phi$, and thus we use $q_\theta$ to infer object labels by default. We summarize the detailed optimization algorithm in Alg.~\ref{alg::optim}. 

Fig.~\ref{fig::framework} presents an illustration of the framework. For the central object, $q_\theta$ uses the attributes of its surrounding objects to learn its representation, and further predicts the label. In contrast, $p_\phi$ utilizes the labels of the surrounding objects as features. If a neighbor is unlabeled, we simply use a label sampled from $q_\theta$ instead. In the E-step, $p_\phi$ predicts the label for the central object, which is then treated as target to update $q_\theta$. In the M-step, $q_\theta$ predicts the label for the central object, which serves as the target data to update $p_\phi$.

\section{Application}

Besides semi-supervised object classification, our approach can also be naturally extended to many other tasks. In this section, we choose two applications for demonstration. 

\subsection{Unsupervised Node Representation Learning}

Though \method{} is designed to learn node representations for semi-supervised node classification, it can be naturally extended to learn node representation without any labeled nodes. In this case, as there are no labeled nodes, we instead predict the neighbors for each node. In this way, the neighbors of each node are treated as the pseudo labels, and this idea is widely used in existing studies~\cite{perozzi2014deepwalk,tang2015line,grover2016node2vec}. Based on that, we first train the inference network $q_\theta$ with the pseudo label of each node, then we continue training with Alg.~\ref{alg::optim}, in which all the nodes labels are treated as unobserved variables. In the E-step, we infer the neighbor distribution for each node with $q_\theta$ by optimizing Eq.~\eqref{eqn::obj-q-u}, during which $p_\phi$ encourages the inferred neighbor distributions to be locally smooth. In the M-step, we update $p_\phi$ with Eq.~\eqref{eqn::obj-p} to model the local dependency of the inferred neighbor distributions.

The inference network $q_\theta$ in the above method shares similar ideas with existing methods for unsupervised node representation learning~\cite{perozzi2014deepwalk,tang2015line,grover2016node2vec}, since they all seek to predict the neighbors for each node. However, we also use $p_\phi$ to ensure that the neighbor distributions inferred by $q_\theta$ are locally smooth, and thus $p_\phi$ serves as a regularizer to improve $q_\theta$.

\subsection{Link Classification}

\method{} can be also applied to link classification. Given a few labeled links, the goal is to classify the remaining links.

Following the idea in existing SRL studies~\cite{taskar2004link}, we reformulate link classification as an object classification problem. Specifically, we construct a line graph $\tilde{G}$ from the original object graph $G$. The object set $\tilde{V}$ in the line graph corresponds to the link set $E$ in the original graph. Two objects are linked in the line graph if their corresponding links in the original graph share a node. The attributes of each node in the line graph is defined as the nodes of the corresponding link in the original graph. 

Based on that, the link classification task on the original graph is formulated as the object classification task on the line graph. Therefore, we can apply our object classification approach to the line graph for solving the original problem.

\section{Experiment}

\begin{table*}[bht]
    \vspace{-0.2cm}
	\caption{Dataset statistics. OC, NRL, LC represent object classification, node representation learning and link classification respectively.}
	\label{tab::dataset}
	\begin{center}
	\scalebox{0.75}
	{
		\begin{tabular}{c c c c c c c c c}\hline
		    \textbf{Dataset} & \textbf{Task} & \textbf{\# Nodes} & \textbf{\# Edges} & \textbf{\# Features} & \textbf{\# Classes} & \textbf{\# Training} & \textbf{\# Validation} & \textbf{\# Test}  \\
	        \hline
	        Cora & OC / NRL & 2,708 & 5,429 & 1,433 & 7 & 140 & 500 & 1,000\\
	        Citeseer & OC / NRL & 3,327 & 4,732 & 3,703 & 6 & 120 & 500 & 1,000 \\
	        Pubmed & OC / NRL & 19,717 & 44,338 & 500 & 3 & 60 & 500 & 1,000\\
	        Bitcoin Alpha & LC & 3,783 & 24,186 & 3,783 & 2 & 100 & 500 & 3,221\\
	        Bitcoin OTC & LC & 5,881 & 35,592 & 5,881 & 2 & 100 & 500 & 5,947\\
	        \hline
	    \end{tabular}
	}
	\end{center}
	\vspace{-0.5cm}
\end{table*}

\begin{table}[bht]
    \vspace{-0.2cm}
	\caption{Results of object classification (\%). [*] means the results are taken from the corresponding papers.} 
	\label{tab::results-object}
	\begin{center}
	\scalebox{0.75}
	{
		\begin{tabular}{C{1.4cm} C{2.3cm} C{1.5cm} C{1.5cm} C{1.5cm}}\hline
		    \textbf{Category}	& \textbf{Algorithm}	& \textbf{Cora} & \textbf{Citeseer} & \textbf{Pubmed} \\ 
		    \hline
		    \multirow{1}{*}{\textbf{SSL}} 
		    & LP & 74.2 & 56.3 & 71.6 \\
		    \hline
		    \multirow{3}{*}{\textbf{SRL}} 
		    & PRM & 77.0 & 63.4 & 68.3 \\
		    & RMN & 71.3 & 68.0 & 70.7 \\
		    & MLN & 74.6 & 68.0 & 75.3 \\
		    \hline
		    \multirow{3}{*}{\textbf{GNN}} 
		    & Planetoid * & 75.7 & 64.7 & 77.2 \\
		    & GCN * & 81.5 & 70.3 & 79.0 \\
		    & GAT * & 83.0 & 72.5 & 79.0 \\
		    \hline
		    \multirow{3}{*}{\textbf{\method{}}}
		    & W/o Attr. in $p_\phi$ & 83.4 & 73.1 & 81.4 \\
		    & With Attr. in $p_\phi$ & \textbf{83.7} & 72.9 & 81.8 \\
		    & Best results & \textbf{83.7} & \textbf{73.6} & \textbf{81.9} \\
		    \hline
	    \end{tabular}
	}
	\end{center}
	\vspace{-0.5cm}
\end{table}

\begin{table}[bht]
	\caption{Results of unsupervised node representation learning (\%). [*] means the results are taken from corresponding papers.}
	\label{tab::results-unsup}
	\begin{center}
	\scalebox{0.75}
	{
		\begin{tabular}{C{1.4cm} C{2.3cm} C{1.5cm} C{1.5cm} C{1.5cm}}\hline
		    \textbf{Category}	& \textbf{Algorithm}	& \textbf{Cora} & \textbf{Citeseer} & \textbf{Pubmed} \\ 
		    \hline
		    \multirow{2}{*}{\textbf{GNN}} 
		    & DeepWalk * & 67.2 & 43.2 & 65.3 \\
		    & DGI * & 82.3 & \textbf{71.8} & 76.8 \\
		    \hline
		    \multirow{2}{*}{\textbf{\method{}}} 
		    & With only $q_{\theta}$ & 78.1 & 68.0 & 79.3 \\
		    & With $q_{\theta}$ and $p_{\phi}$ & \textbf{82.8} & 71.5 & \textbf{81.6} \\
		    \hline
	    \end{tabular}
	}
	\end{center}
	\vspace{-0.5cm}
\end{table}

In this section, we evaluate the performance of \method{} on three tasks, including object classification, unsupervised node representation learning, and link classification. 

\subsection{Datasets and Experiment Settings}

For object classification, we follow existing studies~\cite{yang2016revisiting,kipf2016semi,velivckovic2018graph} and use three benchmark datasets from~\citet{sen2008collective} for evaluation, including Cora, Citeseer, Pubmed. In each dataset, 20 objects from each class are treated as labeled objects, and we use the same data partition as in~\citet{yang2016revisiting}. Accuracy is used as the evaluation metric.

For unsupervised node representation learning, we also use the above three datasets, in which objects are treated as nodes. We learn node representations without using any labeled nodes. To evaluate the learned representations, we follow~\citet{velickovic2018deep} and treat the representations as features to train a linear classifier on the labeled nodes. Then we classify the test nodes and report the accuracy. Note that we use the same data partition as in object classification. 

For link classification, we construct two datasets from the Bitcoin Alpha and the Bitcoin OTC datasets~\cite{kumar2016edge,kumar2018rev2} respectively. The datasets contain graphs between Bitcoin users, and the weight of a link represents the trust degree of connected users. We treat links with weights greater than 3 as positive instances, and links with weights less than -3 are treated as negative ones. Given a few labeled links, We try to classify the test links. As the positive and negative links are quite unbalanced, we report the F1 score. 

\subsection{Compared Algorithms}

\smallskip
\noindent \textbf{GNN Methods}. For object classification and link classification, we mainly compare with the recently-proposed Graph Convolutional Network~\cite{kipf2016semi} and Graph Attention Network~\cite{velivckovic2018graph}. However, GAT cannot scale up to both datasets in the link classification task, so the performance is not reported. For unsupervised node representation learning, we compare with Deep Graph Infomax~\cite{velickovic2018deep}, which is the state-of-the-art method. Besides, we also compare with DeepWalk~\cite{perozzi2014deepwalk} and Planetoid~\cite{yang2016revisiting}.

\smallskip
\noindent \textbf{SRL Methods}. For SRL methods, we compare with the Probabilistic Relational Model~\cite{koller1998probabilistic}, the Relational Markov Network~\cite{taskar2002discriminative} and the Markov Logic Network~\cite{richardson2006markov}. In the PRM method, we assume that the label of an object depends on the attributes of itself, its adjacent objects, and the labels of its adjacent objects. For the RMN and MLN methods, we follow~\citet{taskar2002discriminative} and use a logistic regression model locally for each object. This logistic regression model takes the attributes of each object and also those of its neighbors as features. Besides, we treat the labels of two linked objects as a clique template, which is the same as in~\citet{taskar2002discriminative}. In RMN, a complete score table is employed for modeling label dependency, which maintains a potential score for every possible combination of object labels in a clique. In MLN, we simply use one indicator function in the potential function, and the indicator function judges whether the objects in a clique have the same label. Loop belief propagation~\cite{murphy1999loopy} is used for approximation inference in RMN and MLN.

\smallskip
\noindent \textbf{SSL Methods}. For methods under the category of graph-based semi-supervised classification, we choose the label propagation method~\cite{zhou2004learning} to compare with.

\subsection{Parameter Settings}

\smallskip
\noindent \textbf{Object Classification.} For \method{}, $p_\phi$ and $q_\theta$ are composed of two graph convolutional layers with 16 hidden units and ReLU activation~\cite{nair2010rectified}, followed by the softmax function, as in~\citet{kipf2016semi}. We reweight each feature of a node to 1 if the original weight is greater than 0. Dropout~\cite{srivastava2014dropout} is applied to the network inputs with $p=0.5$. We use the RMSProp optimizer~\cite{tieleman2012lecture} during training, with the initial learning rate as 0.05 and weight decay as 0.0005. In each iteration, both networks are trained for 100 epochs. The mean accuracy over 100 runs is reported in experiment. 

\smallskip
\noindent \textbf{Unsupervised Node Representation Learning.} For the \method{} approach, $p_\phi$ and $q_\theta$ are composed of two graph convolutional layers followed by a linear layer and the softmax function. The dimension of hidden layers is set as 512 for Cora and Citeseer, and 256 for Pubmed, which are the same as in~\citet{velickovic2018deep}. ReLU~\cite{nair2010rectified} is used as the activation function. Each node feature is reweighted to 1 if the original weight is larger than 0. We apply dropout~\cite{srivastava2014dropout} to the inputs of both networks with $p=0.5$. The Adam SGD optimizer~\cite{kingma2014adam} is used for training, with initial learning rate as 0.1 and weight decay as 0.0005. We empirically train $q_\theta$ for 200 epoches during pre-training. Afterwards, we train both $p_\phi$ and $q_\theta$ for 2 iterations, with 100 epochs for each network per iteration. For the Pubmed dataset, we transform the raw node features into binary value since it can result in better performance. The mean accuracy over 50 runs is reported. 

\smallskip
\noindent \textbf{Link Classification.} The setting of \method{} in this task is similar as in object classification, with the following differences. The dimension of the hidden layers is set as 128. No weight decay and dropout are used. In each iteration, both networks are trained for 5 epochs with the Adam optimizer~\cite{kingma2014adam}, and the learning rate is 0.01.

\subsection{Results}

\smallskip
\noindent \textbf{1. Comparison with the Baseline Methods.}
\label{sec::experiment-main}
The quantitative results on the three tasks are presented in Tab.~\ref{tab::results-object},~\ref{tab::results-unsup},~\ref{tab::results-link} respectively.
For object classification, \method{} significantly outperforms all the SRL methods. The performance gain is from two folds. First, during inference, \method{} employs a GNN model, which can learn effective object representations to improve inference. Second, during learning, we model the local label dependency with another GNN, which is more effective compared with SRL methods. \method{} is also superior to the label propagation method, as \method{} is able to use object attributes and propagate labels in a non-linear way. Compared with GCN, which employs the same architecture as the inference network in \method{}, \method{} significantly outperforms GCN, and the performance gain mainly comes from the capability of modeling label dependencies. Besides, \method{} also outperforms GAT, but their performances are quite close. This is because GAT utilizes a much more complicated architecture. Since GAT is less efficient, it is not used in \method{}, but we anticipate the results can be further improved by using GAT, and we leave it as future work. In addition, by incorporating the object attributes in the learning network $p_\phi$, we further improve the performance, showing that GMNN is flexible and also effective to incorporate additional features in the learning network. For link classification, we obtain similar results.

For unsupervised node representation learning, \method{} achieves state-of-the-art results on the Cora and Pubmed datasets. The reason is that \method{} effectively models the smoothness of the neighbor distributions for different nodes with the $p_{\phi}$ network. Besides, the performance of \method{} is quite close to the performance in the semi-supervised setting (Tab.~\ref{tab::results-object}), showing that the learned representations are quite effective. We also compare with a variant without using the $p_{\phi}$ network (with only $q_{\theta}$). In this case, we see that the performance drops significantly, showing the importance of using $p_{\phi}$ as a regularizer over the neighbor distributions.

\begin{table}[bht]
	\caption{Results of link classification (\%).}
	\label{tab::results-link}
	\begin{center}
	\scalebox{0.75}
	{
		\begin{tabular}{C{1.4cm} c C{2.3cm} C{2.3cm}}\hline
		    \textbf{Category}	& \textbf{Algorithm}	& \textbf{Bitcoin Alpha} & \textbf{Bitcoin OTC}  \\ 
		    \hline
		    \multirow{1}{*}{\textbf{SSL}} 
		    & LP & 59.68 & 65.58 \\
		    \hline
		    \multirow{3}{*}{\textbf{SRL}} 
		    & PRM & 58.59 & 64.37 \\
		    & RMN & 59.56 & 65.59 \\
		    & MLN & 60.87 & 65.62 \\
		    \hline
		    \multirow{2}{*}{\textbf{GNN}} 
		    & DeepWalk & 62.71 & 63.20 \\
		    & GCN & 64.00 & 65.69 \\
		    \hline
		    \multirow{2}{*}{\textbf{\method{}}} 
		    & W/o Attr. in $p_\phi$ & 65.59 & 66.62 \\
		    & With Attr. in $p_\phi$ & \textbf{65.86} & \textbf{66.83} \\
		    \hline
	    \end{tabular}
	}
	\end{center}
\end{table}

\smallskip
\noindent \textbf{2. Analysis of the Amortized Inference.}
In \method{}, we employ amortized inference, and parameterize the posterior label distribution by using a GNN model. In this section, we thoroughly look into this strategy, and present some analysis in Tab.~\ref{tab::results-amortized}. Here, the variant ``Non-amortized'' simply models each $q_\theta(\mathbf{y}_n|\mathbf{x}_V)$ as a categorical distribution with independent parameters, and performs fix-point iteration (i.e. Eq.~\eqref{eqn::optimal-condition}) to calculate the value. We see that the performance of this variant is very poor on all datasets. By parameterizing the posterior distribution as a neural network, which leverages the own attributes of each object for inference, the performance (see ``1 Linear Layer'') is significantly improved, but still not satisfactory. With several GC layers, we are able to incorporate the attributes from the surrounding neighbors for each object, yielding further significant improvement. The above observations prove the effectiveness of our strategy for inferring the posterior label distributions.

\begin{table}[bht]
	\caption{Analysis of amortized inference (\%).}
	\label{tab::results-amortized}
	\begin{center}
	\scalebox{0.8}
	{
		\begin{tabular}{c C{1.3cm} C{1.3cm} C{1.3cm}}\hline
		    \textbf{Architecture}	& \textbf{Cora} & \textbf{Citeseer} & \textbf{Pubmed} \\ 
		    \hline
		    Non-amortized & 45.3 & 28.1 & 42.2  \\
		    1 Linear Layer & 55.8 & 57.5 & 69.8  \\
		    1 GC Layer & 72.9 & 67.6 & 71.8  \\
		    2 GC Layers & 83.4& 73.1 & 81.4 \\
		    3 GC Layers & 82.0 & 70.6 & 80.7 \\
		    \hline
	    \end{tabular}
	}
	\end{center}
\end{table}

\smallskip
\noindent \textbf{3. Ablation Study of the Learning Network.}
\label{sec::experiment-learn}
In \method{}, the conditional distribution $p_\phi(\mathbf{y}_n|\mathbf{y}_{\text{NB(n)}},\mathbf{x}_V)$ is parameterized as another GNN, which essentially models the local label dependency. In this section, we compare different architectures of the GNN on the object classification task, and the results are presented in Tab.~\ref{tab::results-learn}. Here, the variant ``1 Mean Pooling Layer'' computes the distribution of $\mathbf{y}_n$ as the linear combination of $\{ \mathbf{y}_{n'} \}_{n' \in \text{NB(n)}}$. This variant is similar to label propagation methods, and its performance is quite competitive. However, the weights of different neighbors during propagation are fixed. By parameterizing the conditional distribution with several GC layers, we are able to automatically learn the propagation weights, and thus obtain superior results on all datasets. This observation proves the effectiveness of employing GNNs in the learning procedure.

\begin{table}[bht]
	\caption{Ablation study of the learning network (\%).}
	\label{tab::results-learn}
	\begin{center}
	\scalebox{0.8}
	{
		\begin{tabular}{c C{1.3cm} C{1.3cm} C{1.3cm}}\hline
		    \textbf{Architecture}	& \textbf{Cora} & \textbf{Citeseer} & \textbf{Pubmed} \\ 
		    \hline
		    1 Mean Pooling Layer & 82.4 & 71.9 & 80.7  \\
		    1 GC Layer & 83.1 & 73.1 & 80.9 \\
		    2 GC Layers & 83.4& 73.1 & 81.4 \\
		    3 GC Layers & 83.6 & 73.0 & 81.5 \\
		    \hline
	    \end{tabular}
	}
	\end{center}
	\vspace{-0.2cm}
\end{table}

\smallskip
\noindent \textbf{4. Convergence Analysis.}
In \method{}, we utilize the variational EM algorithm for optimization, which consists of an E-step and an M-step in each iteration. Next, we analyze the convergence of \method{}. We take the Cora and Citeseer datasets on object classification as examples, and report the validation accuracy of both the $q_\theta$ and $p_\phi$ networks at each iteration. Fig.~\ref{fig::conv} presents the convergence curve, in which iteration 0 corresponds to the pre-training stage. \method{} takes only few iterations to convergence, which is very efficient.

\begin{figure}[htb!]
    \vspace{-0.5cm}
	\centering
	\scalebox{1.0}{
	\subfigure[Cora]{
		\label{fig::conv-cora}
		\includegraphics[width=0.24\textwidth]{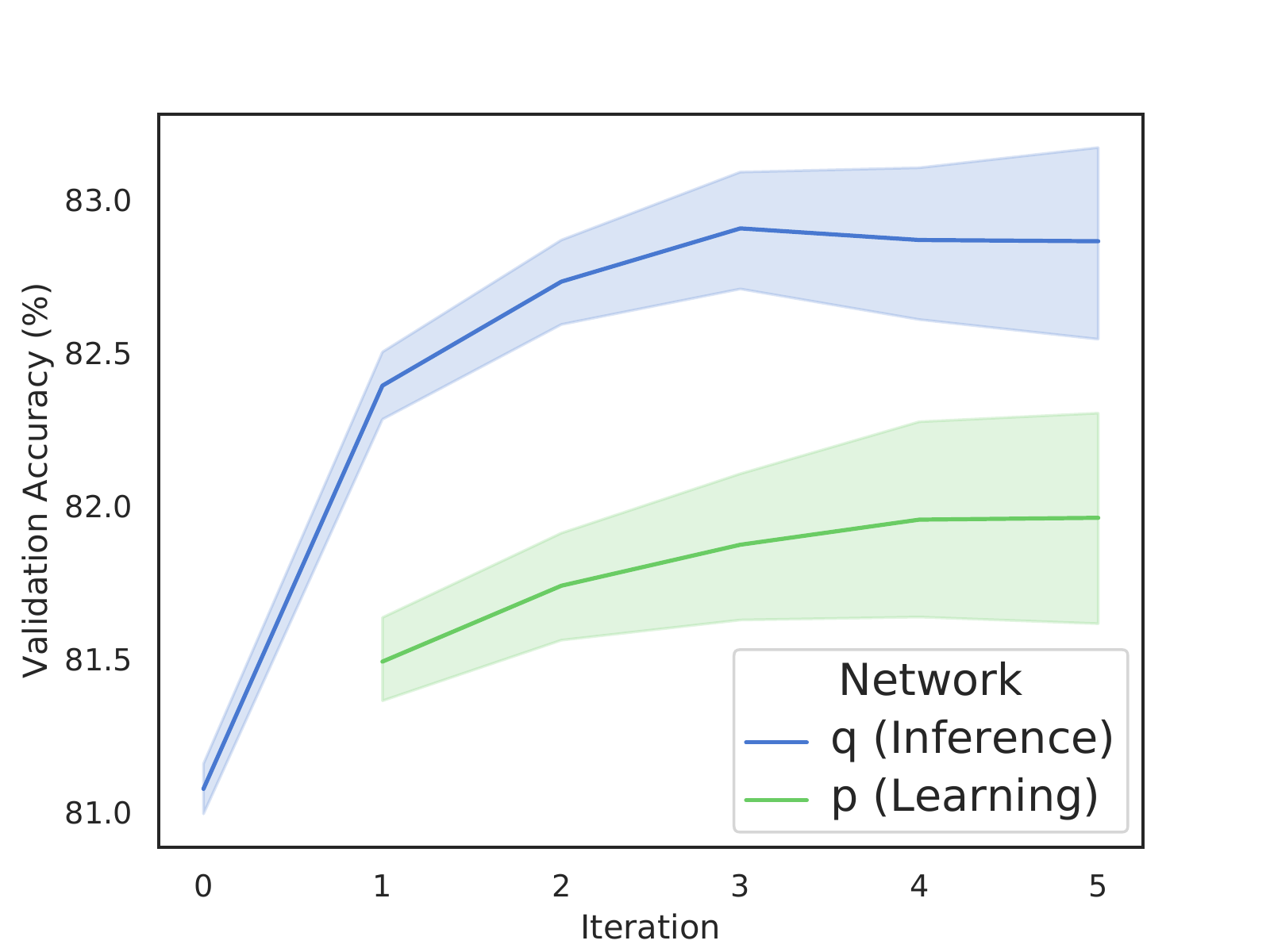}
	}
	\hspace{-0.25cm}
	\subfigure[Citeseer]{
		\label{fig::conv-citeseer}
		\includegraphics[width=0.24\textwidth]{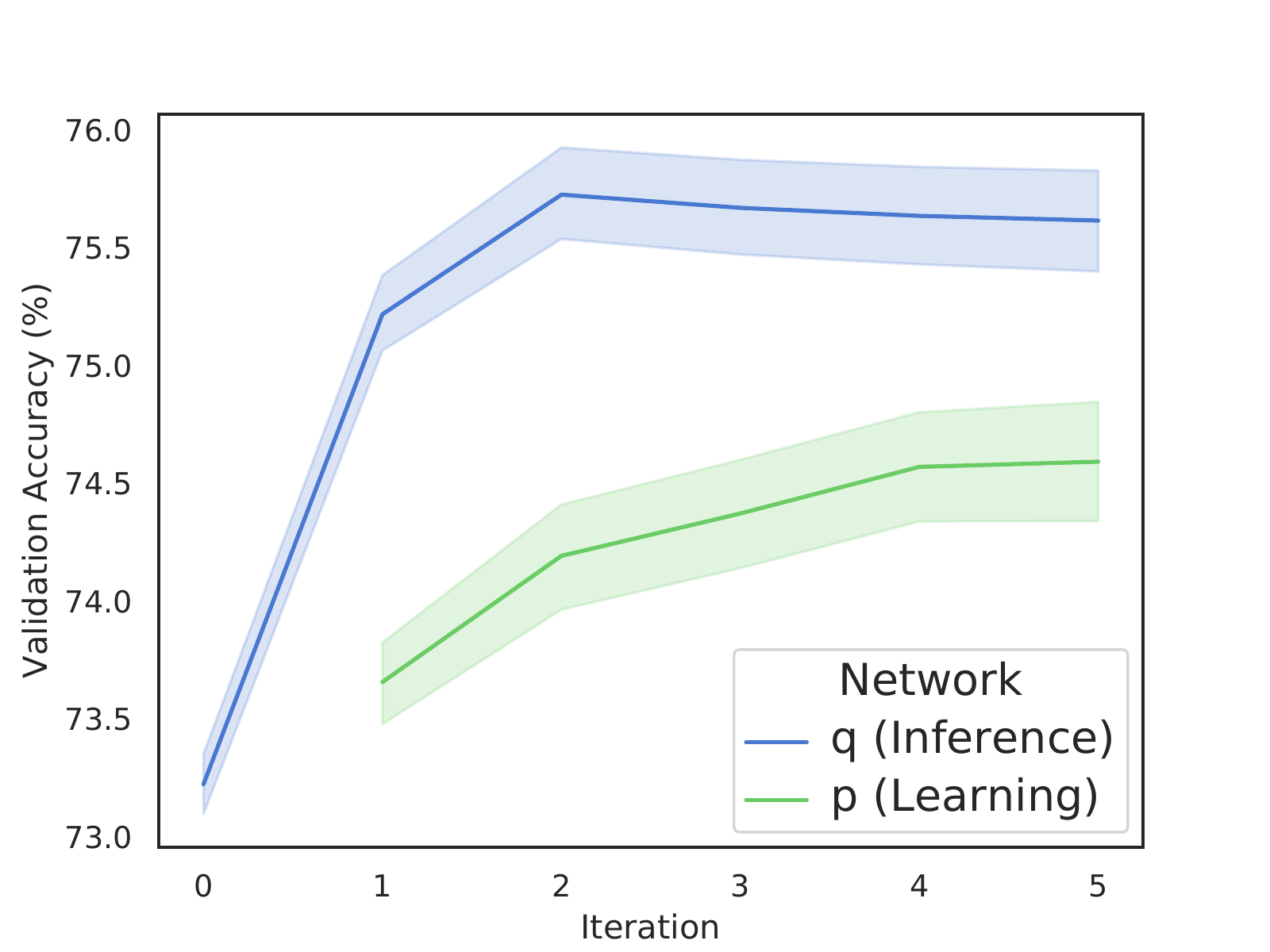}
	}
	}
	\vspace{-0.3cm}
	\caption{Convergence analysis.}
	\label{fig::conv}
	\vspace{-0.3cm}
\end{figure}

\section{Conclusion}

This paper studies semi-supervised object classification, which is a fundamental problem in relational data modeling, and a novel approach called the \method{} is proposed. \method{} employs a conditional random field to model the joint distribution of object labels, and two graph neural networks are utilized to improve both the inference and learning procedures. Experimental results on three tasks prove the effectiveness of \method{}. In the future, we plan to further improve \method{} to deal with graphs with multiple edge types, such as knowledge graphs~\cite{bollacker2008freebase}.

\section*{Acknowledgements}
We would like to thank all the anonymous reviewers for their insightful comments. We also thank Carlos Lassance, Mingzhe Wang, Zhaocheng Zhu, Weiping Song for reviewing the paper before submission. Jian Tang is supported by the Natural Sciences and Engineering Research Council of Canada, as well as the Canada CIFAR AI Chair Program.

\bibliography{paper}
\bibliographystyle{icml2018}

\clearpage

\appendix

\onecolumn

\section{Optimality Condition for $q_\theta$}

\begin{theorem}
Given an object $n$, consider a fixed variational distribution $q_\theta(\mathbf{y}_{\text{NB}(n) \cap U}|\mathbf{x}_V)$ for nodes in $\text{NB}(n) \cap U$. Then the optimum of $q_\theta(\mathbf{y}_n|\mathbf{x}_V)$, denoted by $q^*(\mathbf{y}_n|\mathbf{x}_V)$, is characterized by the following condition:
\begin{equation}\nonumber
\begin{aligned}
&\log q^*(\mathbf{y}_n|\mathbf{x}_V) = \mathbb{E}_{ q_\theta(\mathbf{y}_{\text{NB}(n) \cap U}|\mathbf{x}_V)}[\log p_\phi(\mathbf{y}_n|\mathbf{y}_{\text{NB}(n)},\mathbf{x}_V)]
+ \text{const}.
\end{aligned}
\end{equation}
\end{theorem}

\begin{proof}

To make the notation more concise, we will omit $\mathbf{x}_V$ in the following proof (e.g. simplifying $q_\theta(\mathbf{y}_n|\mathbf{x}_V)$ as $q_\theta(\mathbf{y}_n)$). Recall that our overall goal for $q_\theta(\mathbf{y}_n)$ is to minimize the KL divergence between $q_\theta(\mathbf{y}_U)$ and $p_\phi(\mathbf{y}_U|\mathbf{y}_L)$. Therefore, the objective function for $q_\theta(\mathbf{y}_n)$ could be formulated as follows:

\begin{equation}\nonumber
\label{eqn::appendix-1}
\begin{aligned}
O(q_\theta(\mathbf{y}_{n})) =& -\text{KL}(q_\theta(\mathbf{y}_U)||p_\phi(\mathbf{y}_U|\mathbf{y}_L)) \\
=&\sum_{\mathbf{y}_U} q_\theta(\mathbf{y}_U)[ \log p_\phi(\mathbf{y}_U|\mathbf{y}_L) - \log q_\theta(\mathbf{y}_U)]
\\
=&\sum_{\mathbf{y}_U} \left(\prod_{n'} q_\theta(\mathbf{y}_{n'})\right)\left[\log p_\phi(\mathbf{y}_U,\mathbf{y}_L) - \sum_{n'} \log q_\theta(\mathbf{y}_{n'})\right] + \text{const}
\\
=&\sum_{\mathbf{y}_{n}} \sum_{\mathbf{y}_{U \setminus n}} \left(q_\theta(\mathbf{y}_{n}) \prod_{n' \neq n} q_\theta(\mathbf{y}_{n'})\right)\left[\log p_\phi(\mathbf{y}_U,\mathbf{y}_L) - \sum_{n'} \log q_\theta(\mathbf{y}_{n'})\right] + \text{const}
\\
=&\sum_{\mathbf{y}_{n}} q_\theta(\mathbf{y}_{n}) \sum_{\mathbf{y}_{U \setminus n}} \prod_{n' \neq n} q_\theta(\mathbf{y}_{n'}) \log p_\phi(\mathbf{y}_U,\mathbf{y}_L) - \\
&\sum_{\mathbf{y}_{n}} q_\theta(\mathbf{y}_{n}) \sum_{\mathbf{y}_{U \setminus n}} \prod_{n' \neq n} q_\theta(\mathbf{y}_{n'}) \left[\sum_{n' \neq n} \log q_\theta(\mathbf{y}_{n'}) + \log q_\theta(\mathbf{y}_{n})\right] + \text{const}
\\
=&\sum_{\mathbf{y}_{n}} q_\theta(\mathbf{y}_{n}) \log \mathcal{F}(\mathbf{y}_{n}) - \sum_{\mathbf{y}_{n}} q_\theta(\mathbf{y}_{n}) \log q_\theta(\mathbf{y}_{n}) + \text{const} \\
=& -\text{KL}\left(q_\theta(\mathbf{y}_{n})||\frac{\mathcal{F}(\mathbf{y}_{n})}{Z}\right) + \text{const}.
\end{aligned}
\end{equation}
Here, $Z$ is a normalization term, which makes $\mathcal{F}(\mathbf{y}_{n})$ a valid distribution on $\mathbf{y}_{n}$, and we have:
\begin{equation}\nonumber
\begin{aligned}
\log \mathcal{F}(\mathbf{y}_{n}) &= \sum_{\mathbf{y}_{U \setminus n}} \prod_{n' \neq n} q_\theta(\mathbf{y}_{n'}) \log p_\phi(\mathbf{y}_U,\mathbf{y}_L) = \mathbb{E}_{q_\theta(\mathbf{y}_{U \setminus n})} [\log p_\phi(\mathbf{y}_U,\mathbf{y}_L)].
\end{aligned}
\end{equation}
Based on the above mathematical manipulation of the objective function $O(q_\theta(\mathbf{y}_{n}))$, if a local optimal of $q_\theta(\mathbf{y}_{n})$ is denoted by $q^*(\mathbf{y}_n|\mathbf{x}_V)$, then $q^*(\mathbf{y}_n|\mathbf{x}_V)$ must be equal to $\frac{\mathcal{F}(\mathbf{y}_{n})}{Z}$, and thus we have:
\begin{equation}\nonumber
\begin{aligned}
\log q^*(\mathbf{y}_{n}) &= \log \mathcal{F}(\mathbf{y}_{n}) + \text{const} \\ &=\mathbb{E}_{q_\theta(\mathbf{y}_{U \setminus n})} [\log p_\phi(\mathbf{y}_U,\mathbf{y}_L)] + \text{const} \\
&=\mathbb{E}_{q_\theta(\mathbf{y}_{U \setminus n})} [\log p_\phi(\mathbf{y}_{n}|\mathbf{y}_{V \setminus n})] + \text{const} \\
&=\mathbb{E}_{q_\theta(\mathbf{y}_{U \setminus n})} [\log p_\phi(\mathbf{y}_{n}|\mathbf{y}_{\text{NB}(n)})] + \text{const} \\
&=\mathbb{E}_{q_\theta(\mathbf{y}_{\text{NB}(n) \cap U})} [\log p_\phi(\mathbf{y}_{n}|\mathbf{y}_{\text{NB}(n)})] + \text{const}.
\end{aligned}
\end{equation}
Here, $p_\phi(\mathbf{y}_{n}|\mathbf{y}_{V \setminus n})=p_\phi(\mathbf{y}_{n}|\mathbf{y}_{\text{NB}(n)})$ is based on the conditional independence property of Markov networks.
\end{proof}

\section{Additional Experiment}

\subsection{Results on Random Data Splits}

\begin{table}[bht]
	\caption{Results on random data splits (\%).}
	\label{tab::results-random}
	\begin{center}
	\scalebox{0.9}
	{
		\begin{tabular}{c C{1.3cm} C{1.3cm} C{1.3cm}}\hline
		    \textbf{Algorithm}	& \textbf{Cora} & \textbf{Citeseer} & \textbf{Pubmed} \\ 
		    \hline
		    GCN  & 81.5 & 71.3 & 80.3 \\
		    GAT  & 82.1 & 71.5 & 80.1 \\
		    GMNN & \textbf{83.1} & \textbf{73.0} & \textbf{81.9}  \\
		    \hline
	    \end{tabular}
	}
	\end{center}
\end{table}

In the previous experiment, we have seen that \method{} significantly outperforms all the baseline methods for semi-supervised object classification under the data splits from~\citet{yang2016revisiting}. To further validate the effectiveness of \method{}, we also evaluate \method{} on some random data splits. Specifically, we randomly create 10 data splits for each dataset. The size of the training, validation and test sets in each split is the same as the split in~\citet{yang2016revisiting}. We compare \method{} with GCN~\cite{kipf2016semi} and GAT~\cite{velivckovic2018graph} on those random data splits, as they are the most competitive baseline methods. For each data split, we run each method with 10 different seeds, and report the overall mean accuracy in Tab.~\ref{tab::results-random}. We see \method{} consistently outperforms GCN and GAT on all datasets, proving the effectiveness of \method{}.

\subsection{Results on Few-shot Learning Settings}

\begin{table}[bht]
	\caption{Results on few-shot learning settings (\%).}
	\label{tab::results-few}
	\begin{center}
	\scalebox{0.9}
	{
		\begin{tabular}{c C{1.3cm} C{1.3cm} C{1.3cm}}\hline
		    \textbf{Algorithm}	& \textbf{Cora} & \textbf{Citeseer} & \textbf{Pubmed} \\ 
		    \hline
		    GCN  & 74.9 & 69.0 & 76.9 \\
		    GAT  & 77.0 & 68.9 & 75.4 \\
		    GMNN & \textbf{78.6} & \textbf{72.7} & \textbf{79.1}  \\
		    \hline
	    \end{tabular}
	}
	\end{center}
\end{table}

In the previous experiment, we have proved the effectiveness of \method{} for object classification in the semi-supervised setting. Next, we further conduct experiment in the few-short learning setting to evaluate the robustness of \method{} to data sparsity. We choose GCN and GAT for comparison. For each dataset, we randomly sample 5 labeled nodes under each class as training data, and run each method with 100 different seeds. The mean accuracy is shown in Tab.~\ref{tab::results-few}. We see \method{} significantly outperforms GCN and GAT. The improvement is even larger than the case of semi-supervised setting, where 20 labeled nodes under each class are used for training. This result proves that \method{} is robust to the sparsity of training data.

\subsection{Comparison with Self-training Methods}

\begin{table}[bht]
	\caption{Comparison with self-training methods (\%).}
	\label{tab::results-self}
	\begin{center}
	\scalebox{0.9}
	{
		\begin{tabular}{c C{1.3cm} C{1.3cm} C{1.3cm}}\hline
		    \textbf{Algorithm}	& \textbf{Cora} & \textbf{Citeseer} & \textbf{Pubmed} \\ 
		    \hline
		    Self-training  & 82.7 & 72.4 & 80.1 \\
		    GMNN & \textbf{83.4} & \textbf{73.1} & \textbf{81.4}  \\
		    \hline
	    \end{tabular}
	}
	\end{center}
\end{table}

Our proposed \method{} approach is related to self-training frameworks. In \method{}, the $p_\phi$ network essentially tries to annotate unlabeled objects, and the annotated objects are further treated as extra data to update $q_\theta$ through Eq.~\eqref{eqn::obj-q-u}. Similarly, in self-training, we typically use $q_\theta$ itself to annotate unlabeled objects, and collect extra training data for $q_\theta$. Next, we compare \method{} with the self-training method for semi-supervised object classification, and the results are presented in Tab.~\ref{tab::results-self}.

We see \method{} consistently outperforms the self-training method. The reason is that the self-training method uses $q_\theta$ for both inference and annotation, while \method{} uses two different networks $q_\theta$ and $p_\phi$ to collaborate with each other. The information captured by $q_\theta$ and $p_\phi$ is complementary, and therefore \method{} achieves much better results.

\subsection{Comparison of Different Approximation Methods}

\begin{table}[bht]
	\caption{Comparison of different approximation methods (\%).}
	\label{tab::results-approximation}
	\begin{center}
	\scalebox{0.9}
	{
		\begin{tabular}{c C{1.3cm} C{1.3cm} C{1.3cm}}\hline
		    \textbf{Method}	& \textbf{Cora} & \textbf{Citeseer} & \textbf{Pubmed} \\ 
		    \hline
		    Single Sample & 82.1 & 71.5 & 80.4 \\
		    Multiple Samples & 83.2 & 72.5 & 81.1 \\
		    Annealing & \textbf{83.4}& \textbf{73.1} & \textbf{81.4} \\
		    Max Pooling & 83.2 & 72.8 & 81.2  \\
		    Mean Pooling & \textbf{83.4} & 72.6 & 80.5 \\
		    \hline
	    \end{tabular}
	}
	\end{center}
\end{table}

In \method{}, we use a mean-field variational distribution $q_\theta$ for inference, and the optimal $q_\theta$ is given by the fixed-point condition in Eq.~\eqref{eqn::optimal-condition}. Learning the optimal $q_\theta$ requires computing the right-hand side of Eq.~\eqref{eqn::optimal-condition}, which involves the expectation with respect to $q_\theta(\mathbf{y}_{\text{NB}(n) \cap U}|\mathbf{x}_V)$ for each node $n$. To estimate the expectation, we notice that $q_\theta(\mathbf{y}_{\text{NB}(n) \cap U}|\mathbf{x}_V)$ can be factorized as $\prod_{n' \in \text{NB}(n) \cap U} q_\theta(\mathbf{y}_{n'}|\mathbf{x}_V)$. Based on that, we can develop several empirical approximation methods.

\smallskip
\noindent \textbf{Single Sample.} 
The simplest way is to draw a single sample $\mathbf{\hat{y}}_{n'} \sim q_\theta(\mathbf{y}_{n'}|\mathbf{x}_V)$ for each node $n' \in \text{NB}(n) \cap U$, and then we could use the sample to estimate the expectation.

\smallskip
\noindent \textbf{Multiple Samples.} 
In practice, we can also draw multiple samples from $q_\theta(\mathbf{y}_{n'}|\mathbf{x}_V)$ for each node $n' \in \text{NB}(n) \cap U$ to estimate the expectation. Such a method has lower variance but entails higher cost.

\smallskip
\noindent \textbf{Annealing.} Another method is to introduce an annealing parameter $\tau$ in $q_\theta(\mathbf{y}_{n'}|\mathbf{x}_V)$, so that we have:
\begin{equation}\nonumber
    q_\theta(\mathbf{y}_{n'}|\mathbf{x}_V)=\text{Cat}(\mathbf{y}_{n'}|\text{softmax}(\frac{W_\theta \mathbf{h}_{\theta,n'}}{\tau})).
\end{equation}
Then we can set $\tau$ to a small value (e.g. 0.1) and draw a sample $\mathbf{\hat{y}}_{n'} \sim q_\theta(\mathbf{y}_{n'}|\mathbf{x}_V)$ for $n' \in \text{NB}(n) \cap U$ to estimate the expectation, which typically has lower variance.

\smallskip
\noindent \textbf{Max Pooling.} Another method is max pooling, where we set $\mathbf{\hat{y}}_{n'} =\arg\max_{\mathbf{y}_{n'}} q_\theta(\mathbf{y}_{n'}|\mathbf{x}_V)$ for $n' \in \text{NB}(n) \cap U$, and use $\{\mathbf{\hat{y}}_{n'}\}_{n' \in \text{NB}(n) \cap U}$ as a sample to estimate the expectation.

\smallskip
\noindent \textbf{Mean Pooling.} Besides, we can also use the mean pooling method similar to the soft attention method in~\citet{deng2018latent}. Specifically, suppose that we have $C$ classes for classification, then the label of each node can be viewed as a one-hot $C$-dimensional vector. Based on that, we set $\mathbf{\bar{y}}_{n'} = \mathbb{E}_{q_\theta(\mathbf{y}_{n'}|\mathbf{x}_V)}[\mathbf{y}_{n'}]$ for each unlabeled neighbor $n'$ of the object $n$, which can be understood as a soft label vector of that neighbor. For each labeled neighbor $n'$ of the object $n$, we set $\mathbf{\bar{y}}_{n'}$ as the one-hot label vector. Then we can approximate the expectation as:
\begin{equation}\nonumber
\begin{aligned}
    &\mathbb{E}_{ q_\theta(\mathbf{y}_{\text{NB}(n) \cap U}|\mathbf{x}_V)}[\log p_\phi(\mathbf{y}_n|\mathbf{y}_{\text{NB}(n)},\mathbf{x}_V)] \\ \approx & p_\phi(\mathbf{y}_n|\mathbf{\bar{y}}_{\text{NB}(n)},\mathbf{x}_V)] \defeq \text{Cat}(\mathbf{y}_n|\text{softmax}(W_\phi\mathbf{h}_{\phi,n})), \quad \text{with}\quad \mathbf{h}_{\phi,n} = g(\mathbf{\bar{y}}_{\text{NB}(n)},E),
\end{aligned}
\end{equation}
where the object representation $\mathbf{h}_{\phi,n}$ is learned by feeding $\{ \mathbf{\bar{y}}_{n'} \}_{n' \in \text{NB}(n)}$ as features in a graph neural network $g$.

\smallskip
\noindent \textbf{Comparison.}
We empirically compare different methods in the semi-supervised object classification task, where we use 10 samples for the multi-sample method and the parameter $\tau$ in the annealing method is set as 0.1. Tab.~\ref{tab::results-approximation} presents the results. We see the annealing method consistently outperforms other methods on all datasets, and therefore we use the annealing method for all the experiments in the paper.

\subsection{Best Results with Standard Deviation}

\begin{table}[bht]
	\caption{Best results on semi-supervised object classification (\%).}
	\label{tab::results-best}
	\begin{center}
	\scalebox{0.8}
	{
		\begin{tabular}{c C{2.2cm} C{2.2cm} C{2.2cm}}\hline
		    \textbf{Algorithm}	& \textbf{Cora} & \textbf{Citeseer} & \textbf{Pubmed} \\ 
		    \hline
		    GAT & 83.0 $\pm$ 0.7  & 72.5 $\pm$ 0.7 & 79.0 $\pm$ 0.3  \\
		    GMNN & 83.675 $\pm$ 0.900  & 73.576 $\pm$ 0.795 & 81.922 $\pm$ 0.529  \\
		    \hline
		    $p$-value & $<0.0001$ & $<0.0001$ & $<0.0001$ \\
		    \hline
	    \end{tabular}
	}
	\end{center}
\end{table}

Finally, we present the best mean accuracy together with the standard deviation of \method{} over 100 runs in Tab.~\ref{tab::results-best}. The improvement over GAT is statistically significant.

\end{document}